\newif\ificml
\icmltrue

\documentclass{article}

\usepackage{microtype}
\usepackage{graphicx}
\usepackage{subfigure}
\usepackage{booktabs} 
\usepackage{amsmath}
\usepackage{mathtools}
\usepackage{color}
\usepackage{amsthm, amssymb, bm}
\usepackage[justification=justified,singlelinecheck=false]{caption}
\usepackage[flushleft]{threeparttable}
\usepackage{hyperref}
\usepackage{multirow, makecell}
\usepackage{footnote}
\usepackage[T1]{fontenc}
\usepackage{lmodern}

\DeclareMathOperator*{\argmax}{\arg\hspace{-0.07em}\max}
\DeclareMathOperator*{\argmin}{\arg\hspace{-0.07em}\min}

\ificml
    \usepackage[accepted]{icml2021}
\else 
    \usepackage{natbib}
    \usepackage{algorithm,algorithmic}
\fi

\newcommand{\mf}[1]{\mathbf{#1}}
\newtheorem{thm}{Theorem}
\newtheorem{lem}{Lemma}
\newtheorem{defi}{Definition}

\ificml
    \usepackage{xr}
\fi

\ificml
    \icmltitlerunning{ASAM: Adaptive Sharpness-Aware Minimization for Scale-Invariant Learning of Deep Neural Networks}
\fi

\begin{document}

\ificml
    \twocolumn[
    \icmltitle{ASAM: Adaptive Sharpness-Aware Minimization\\ for Scale-Invariant Learning of Deep Neural Networks}

    \icmlsetsymbol{equal}{*}

    \begin{icmlauthorlist}
    \icmlauthor{Jungmin Kwon}{sr}
    \icmlauthor{Jeongseop Kim}{sr}
    \icmlauthor{Hyunseo Park}{sr}
    \icmlauthor{In Kwon Choi}{sr}
    \end{icmlauthorlist}

    \icmlaffiliation{sr}{Samsung Research, Seoul, Republic of Korea}

    \icmlcorrespondingauthor{Jeongseop Kim}{jisean.kim@samsung.com}

    \icmlkeywords{Machine Learning, ICML}

    \vskip 0.3in
    ]
    \printAffiliationsAndNotice{}  
\else 
    \title{ASAM: Adaptive Sharpness-Aware Minimization\\ for Scale-Invariant Learning of Deep Neural Networks}
    \author{
      Jungmin Kwon$^*$
      \and
      Jeongseop Kim\thanks{Samsung Research, Seoul, Republic of Korea, Correspondence to: Jeongseop Kim $<jisean.kim@samsung.com>$.} \\
      \and
      Hyunseo Park$^*$
      \and
      In Kwon Choi$^*$
    }
    \date{}
    \maketitle
\fi

\ificml
    \newcommand\figureratio{0.85}
    \newcommand\beginfigure{\begin{figure}[!t]}
    \newcommand\jmkendfigure{\end{figure}}
    \newcommand\begintable{\begin{table}[!h]}
    \newcommand\jmkendtable{\end{table}}
\else
    \newcommand\figureratio{0.75}
    \newcommand\beginfigure{\begin{figure}}
    \newcommand\jmkendfigure{\end{figure}}
    \newcommand\begintable{\begin{table}}
    \newcommand\jmkendtable{\end{table}}
\fi

\begin{abstract}

Recently, learning algorithms motivated from sharpness of loss surface as an effective measure of generalization gap have shown state-of-the-art performances. Nevertheless, sharpness defined in a rigid region with a fixed radius, has a drawback in sensitivity to parameter re-scaling which leaves the loss unaffected, leading to weakening of the connection between sharpness and generalization gap. In this paper, we introduce the concept of adaptive sharpness which is scale-invariant and propose the corresponding generalization bound. We suggest a novel learning method, adaptive sharpness-aware minimization (ASAM), utilizing the proposed generalization bound. Experimental results in various benchmark datasets show that ASAM contributes to significant improvement of model generalization performance.

\end{abstract}

\section{Introduction}
\label{introduction}

Generalization of deep neural networks has recently been studied with great importance to address the shortfalls of pure optimization, yielding models with no guarantee on generalization ability. 
To understand the generalization phenomenon of neural networks, many studies have attempted to clarify the relationship between the geometry of the loss surface and the generalization performance \citep{hochreiter1995simplifying, mcallester1999pac, keskar2017large, neyshabur2017exploring, jiang2019fantastic}. 
Among many proposed measures used to derive generalization bounds, loss surface sharpness and 
minimization of the derived generalization bound have proven to be effective in attaining state-of-the-art performances in various tasks \citep{hochreiter1997flat, mobahi2016training, chaudhari2019entropy, DBLP:journals/corr/abs-2006-05620, yue2020salr}.

Especially, Sharpness-Aware Minimization (SAM) \citep{foret2021sharpnessaware} as a learning algorithm based on PAC-Bayesian generalization bound, achieves a state-of-the-art generalization performance for various image classification tasks benefiting from minimizing sharpness of loss landscape, which is correlated with generalization gap. Also, they suggest a new sharpness calculation strategy, which is computationally efficient, since it requires only a single gradient ascent step in contrast to other complex generalization measures such as sample-based or Hessian-based approach. 

However, even sharpness-based learning methods including SAM and some of sharpness measures suffer from sensitivity to model parameter re-scaling. \citet{dinh2017sharp} point out that parameter re-scaling which does not change loss functions can cause a difference in sharpness values so this property may weaken correlation between sharpness and generalization gap. We call this phenomenon scale-dependency problem. 

To remedy the scale-dependency problem of sharpness, many studies have been conducted recently \citep{liang2019fisher, yi2019positively, karakida2019normalization, tsuzuku2020normalized}. However, those previous works are limited to proposing only generalization measures which do not suffer from the scale-dependency problem and do not provide sufficient investigation on combining learning algorithm with the measures. 

To this end, we introduce the concept of normalization operator which is not affected by any scaling operator that does not change the loss function.
The operator varies depending on the way of normalizing, e.g., element-wise and filter-wise.
We then define \textit{adaptive sharpness} of the loss function, sharpness whose maximization region is determined by the normalization operator.
We prove that adaptive sharpness remains the same under parameter re-scaling, i.e., \textit{scale-invariant}.
Due to the scale-invariant property, adaptive sharpness shows stronger correlation with generalization gap than sharpness does.

Motivated by the connection between generalization metrics and loss minimization, we propose a novel learning method, adaptive sharpness-aware minimization (ASAM), which adaptively adjusts maximization regions thus acting uniformly under parameter re-scaling. ASAM minimizes the corresponding generalization bound using adaptive sharpness to generalize on unseen data, avoiding the scale-dependency issue SAM suffers from.

The main contributions of this paper are summarized as follows:

\begin{itemize}
    \item We introduce adaptive sharpness of loss surface which is invariant to parameter re-scaling. 
    In terms of rank statistics, adaptive sharpness shows stronger correlation with generalization than sharpness does, which means that adaptive sharpness is more effective measure of generalization gap.
    \item We propose a new learning algorithm using adaptive sharpness which helps alleviate the side-effect in training procedure caused by scale-dependency by adjusting their maximization region with respect to weight scale. 
    \item We empirically show its consistent improvement of generalization performance on image classification and machine translation tasks using various neural network architectures.
\end{itemize}

The rest of this paper is organized as follows. Section~\ref{sec:p} briefly describes previous sharpness-based learning algorithm. In Section~\ref{sec:as}, we introduce adaptive sharpness which is a scale-invariant measure of generalization gap after scale-dependent property of sharpness is explained. In Section~\ref{sec:a}, ASAM algorithm is introduced in detail using the definition of adaptive sharpness.
In Section~\ref{sec:e}, we evaluate the generalization performance of ASAM for various models and datasets.
We provide the conclusion and future work in Section~\ref{sec:c}.

\section{Preliminary} \label{sec:p}
Let us consider a model $f:X\rightarrow Y$ parametrized by a weight vector $\mf{w}$ and a loss function $l:Y\times Y\rightarrow \mathbb{R}_+$.
Given a sample $S=\{(\mf{x}_1,\mf{y}_1),\ldots,(\mf{x}_n,\mf{y}_n)\}$ drawn from data distribution $D$ with i.i.d condition, the training loss can be defined as $L_S(\mf{w})=\sum_{i=1}^n l(\mf{y}_i,f(\mf{x}_i;\mf{w}))/n$.
Then, the generalization gap between the expected loss $L_D(\mf{w})=\mathbb{E}_{(\mf{x},\mf{y})\sim D}[l(\mf{y},f(\mf{x};\mf{w}))]$ and the training loss $L_S(\mf{w})$ represents the ability of the model to generalize on unseen data.

Sharpness-Aware Minimization (SAM) \citep{foret2021sharpnessaware} aims to minimize the following PAC-Bayesian generalization error upper bound

\begin{equation}\label{eq:p1}
    L_{D}(\mf{w}) \leq \max_{\Vert \bm{\epsilon}\Vert _{p} \leq \rho } L_{S}(\mf{w}+\bm{\epsilon}) + h\left(\frac{\Vert \mf{w}\Vert _{2}^2}{\rho^2}\right)
\end{equation}
for some strictly increasing function $h$.
The domain of max operator, called maximization region, is an $\ell^p$ ball with radius $\rho$ for $p \geq 1$.
Here, sharpness of the loss function $L$ is defined as
\begin{equation}\label{eq:s1}
    \max_{\Vert \bm{\epsilon} \Vert_{2} \leq \rho}L_S(\mf{w}+\bm{\epsilon}) - L_S(\mf{w}).
\end{equation}
Because of the monotonicity of $h$ in Equation~\ref{eq:p1}, it can be substituted by $\ell^2$ weight decaying regularizer, so the sharpness-aware minimization problem can be defined as the following minimax optimization
\[
    \min_{\mf{w}}\max_{\Vert \bm{\epsilon}\Vert _{p} \leq \rho }L_{S}(\mf{w} + \bm{\epsilon}) + \frac{\lambda}{2} \Vert \mf{w}\Vert _2^2
\]
\noindent where $\lambda$ is a weight decay coefficient.

SAM solves the minimax problem by iteratively applying the following two-step procedure for $t=0,1,2,\ldots$ as
\begin{equation} \label{sam}
  \left\{
    \begin{array}{ll}
      \bm{\epsilon}_{t} = \rho \frac {\displaystyle \nabla L_S(\mf{w}_t)} {\displaystyle \Vert \nabla L_S(\mf{w}_t) \Vert_2} \\
      \mf{w}_{t+1} = \mf{w}_t - \alpha_t \left( \nabla L_S(\mf{w}_t + \bm{\epsilon}_{t}) + \lambda \mf{w}_t \right)
    \end{array}
  \right.
\end{equation}
where $\alpha_t$ is an appropriately scheduled learning rate. This procedure can be obtained by a first order approximation of $L_S$ and dual norm formulation as
\begin{align*}
    \bm{\epsilon}_{t} 
    & = \argmax_{\Vert \bm{\epsilon}\Vert _{p} \leq \rho } L_{S}(\mf{w}_t+\bm{\epsilon}) \\
    & \approx \argmax_{\Vert \bm{\epsilon}\Vert _{p} \leq \rho } \bm{\epsilon}^\top \nabla L_S(\mf{w}_t) \\
    & = \rho \operatorname{sign}(\nabla L_S(\mf{w}_t)) \frac {\displaystyle \vert \nabla L_S(\mf{w}_t) \vert^{q-1}} {\displaystyle \Vert \nabla L_S(\mf{w}_t) \Vert^{q-1}_q}
\end{align*}
and
\begin{align*}
    \mf{w}_{t+1} 
    & = \argmin_\mf{w} L_{S}(\mf{w}+\bm{\epsilon}_{t}) + \frac{\lambda}{2} \Vert\mf{w}\Vert^2_2 \\
    & \approx \argmin_\mf{w} \, (\mf{w}-\mf{w}_t)^\top \nabla L_S(\mf{w}_t + \bm{\epsilon}_{t}) + \frac{\lambda}{2} \Vert\mf{w}\Vert^2_2 \\
    & \approx \mf{w}_t - \alpha_t (\nabla L_S(\mf{w}_t + \bm{\epsilon}_{t}) + \lambda \mf{w}_t)
\end{align*}
where $1/p+1/q=1$ and $\vert \cdot \vert$ denotes element-wise absolute value function, and $\operatorname{sign}(\cdot)$ also denotes element-wise signum function.
It is experimentally confirmed that the above two-step procedure produces the best performance when $p=2$, which results in Equation~\ref{sam}.

\noindent As can be seen from Equation~\ref{sam}, SAM estimates the point $\mf{w}_t + \bm{\epsilon}_{t}$ at which the loss is approximately maximized around $\mf{w}_t$ in a rigid region with a fixed radius by performing gradient ascent, and performs gradient descent at $\mf{w}_t$ using the gradient at the maximum point $\mf{w}_t + \bm{\epsilon}_{t}$.

\section{Adaptive Sharpness: Scale-Invariant Measure of Generalization Gap} \label{sec:as}

In \citet{foret2021sharpnessaware}, it is experimentally confirmed that the sharpness defined in Equation~\ref{eq:s1} is strongly correlated with the generalization gap. Also they show that SAM helps to find minima which show lower sharpness than other learning strategies and contributes to effectively lowering generalization error.

However, \citet{dinh2017sharp} show that sharpness defined in the rigid spherical region with a fixed radius can have a weak correlation with the generalization gap due to non-identifiability of rectifier neural networks, whose parameters can be freely re-scaled without affecting its output. 

If we assume that $A$ is a scaling operator on the weight space that does not change the loss function, as shown in Figure~\ref{fig:sphere}, the interval of the loss contours around $A\mf{w}$ becomes narrower than that around $\mf{w}$ but the size of the region remains the same, i.e.,
\[\max_{\Vert \bm{\epsilon} \Vert_{2} \leq \rho} L_S(\mf{w}+\bm{\epsilon}) \neq \max_{\Vert \bm{\epsilon} \Vert_{2} \leq \rho} L_S( A\mf{w}+\bm{\epsilon}).\]

Thus, neural networks with $\mf{w}$ and $A\mf{w}$ can have arbitrarily different values of sharpness defined in Equation~\ref{eq:s1}, although they have the same generalization gaps. This property of sharpness is a main cause of weak correlation between generalization gap and sharpness and we call this scale-dependency in this paper.

To solve the scale-dependency of sharpness, we introduce the concept of adaptive sharpness. Prior to explaining adaptive sharpness, we first define normalization operator. The normalization operator that cancels out the effect of $A$ can be defined as follows.

\begin{defi}[Normalization operator]\label{norm_op}
	Let $\{T_\mf{w}, \mf{w}\in\mathbb{R}^k\}$ be a family of invertible linear operators on $\mathbb{R}^k$.
	Given a weight $\mf{w}$, if $T_{A\mf{w}}^{-1}A=T_\mf{w}^{-1}$ for any invertible scaling operator $A$ on $\mathbb{R}^k$ which does not change the loss function, we say $T_\mf{w}^{-1}$ is a normalization operator of $\mf{w}$.
\end{defi}

Using the normalization operator, we define adaptive sharpness as follows.
\begin{defi}[Adaptive sharpness]
If $T_\mf{w}^{-1}$ is the normalization operator of $\mf{w}$ in Definition \ref{norm_op}, adaptive sharpness of $\mf{w}$ is defined by
\begin{equation}\label{eq:s4}
 \max_{\Vert T^{-1}_\mf{w} \bm{\epsilon}\Vert _{p} \leq \rho } L_{S}(\mf{w}+\bm{\epsilon})- L_S(\mf{w})
\end{equation}
where $1 \leq p \leq \infty$.
\end{defi}

Adaptive sharpness in Equation~\ref{eq:s4} has the following properties.

\beginfigure
\centering    
\subfigure[$\Vert \bm{\epsilon} \Vert_2 \leq \rho$ and $\Vert \bm{\epsilon}' \Vert_2 \leq \rho$ \citep{foret2021sharpnessaware} \label{fig:sphere}]{\includegraphics[width=\figureratio\linewidth]{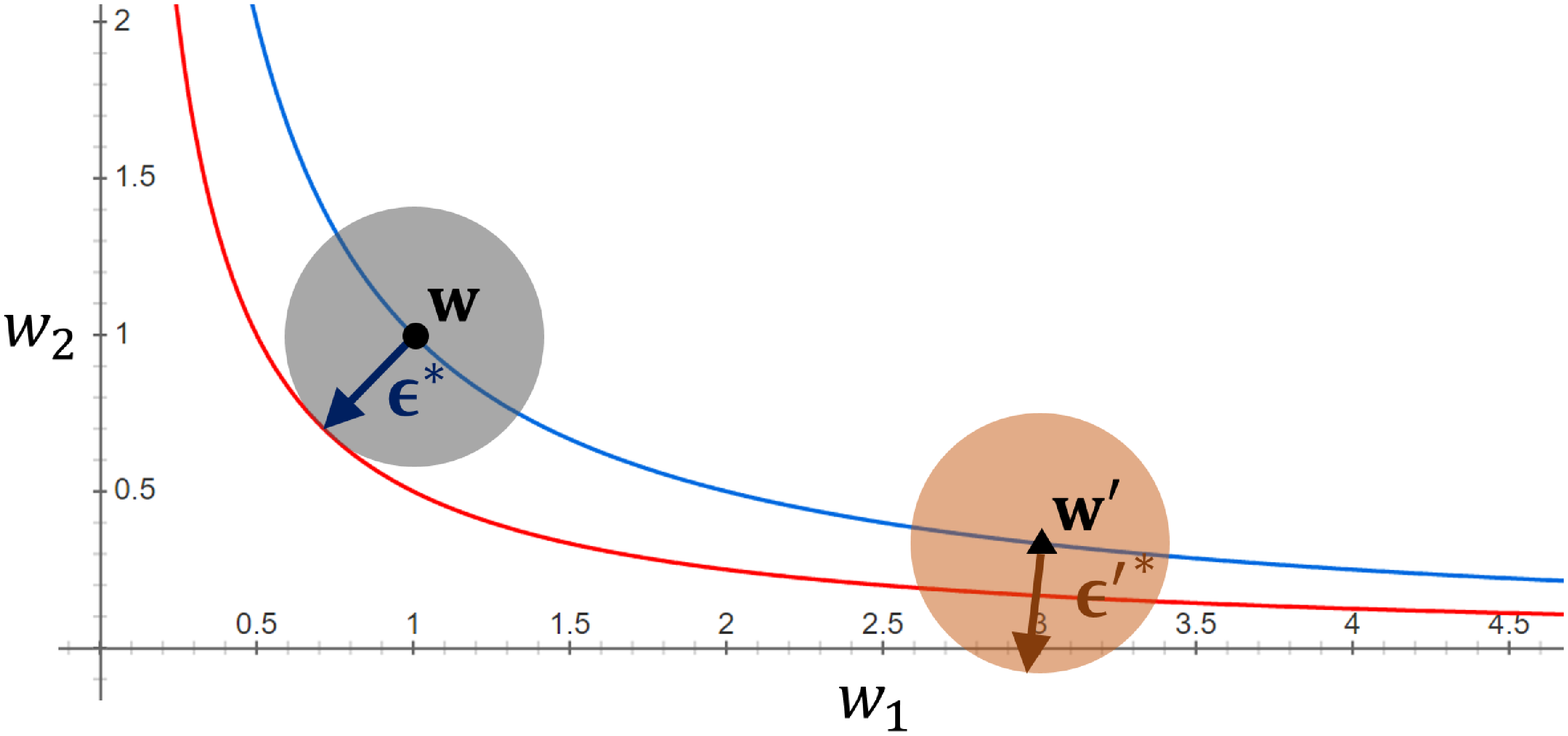}}
\subfigure[$\Vert T^{-1}_\mf{w} \bm{\epsilon} \Vert_\infty \leq \rho$ and $\Vert T^{-1}_{\mf{w}'} \bm{\epsilon}' \Vert_\infty \leq \rho$ \citep{keskar2017large}\label{fig:cuboid}]{\includegraphics[width=\figureratio\linewidth]{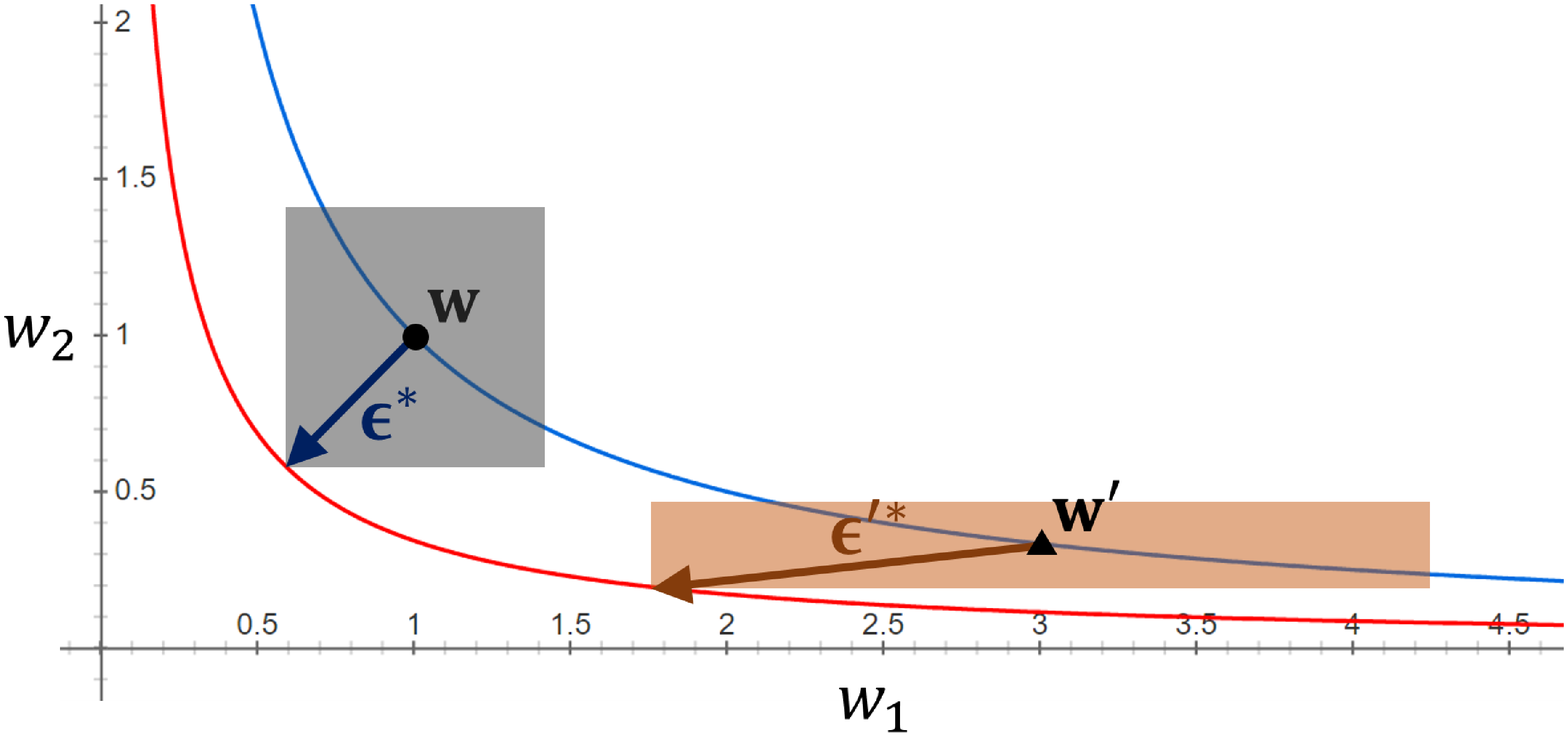}}
\subfigure[$\Vert T^{-1}_\mf{w} \bm{\epsilon} \Vert_2 \leq \rho $ and $\Vert T^{-1}_{\mf{w}'} \bm{\epsilon}' \Vert_2 \leq \rho$ (In this paper)\label{fig:ellipsoid}]{\includegraphics[width=\figureratio\linewidth]{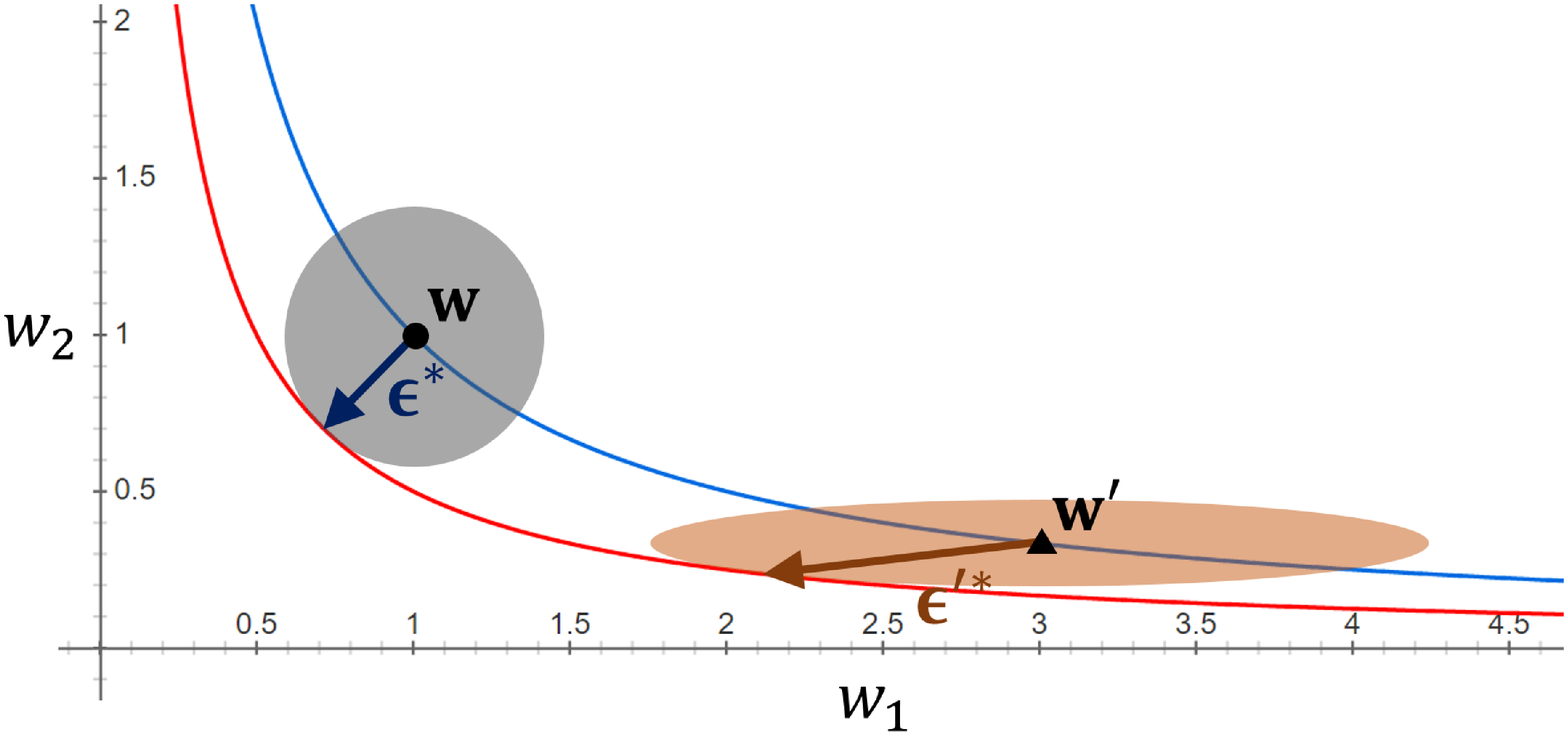}}
\caption{Loss contours and three types of maximization regions: (a) sphere, (b) cuboid and (c) ellipsoid. $\mf{w}=(1,1)$ and $\mf{w}'=(3, 1/3)$ are parameter points before and after multiplying a scaling operator $A=\mathrm{diag}(3, 1/3)$ and are expressed as dots and triangles, respectively. The blue contour line has the same loss at $\mf{w}$, and the red contour line has a loss equal to the maximum value of the loss in each type of region centered on $\mf{w}$. $\bm{\epsilon}^*$ and $\bm{\epsilon'}^*$ are the $\bm{\epsilon}$ and $\bm{\epsilon'}$ which maximize the loss perturbed from $\mf{w}$ and $\mf{w}'$, respectively.}
\label{fig1}
\jmkendfigure

\begin{thm} \label{thm1}
	For any invertible scaling operator $A$ which does not change the loss function, values of adaptive sharpness at $\mf{w}$ and $A\mf{w}$ are the same as
	\begin{align*}
	&\max_{\Vert T_\mf{w}^{-1}\bm{\epsilon} \Vert_p \leq \rho} L_S(\mf{w}+\bm{\epsilon}) - L_S(\mf{w})\\
	&= \max_{\Vert T_{A\mf{w}}^{-1}\bm{\epsilon} \Vert_p \leq \rho} L_S(A \mf{w}+\bm{\epsilon}) - L_S(A \mf{w})
	\end{align*}
	where $T_\mf{w}^{-1}$ and $T_{A\mf{w}}^{-1}$ are the normalization operators of $\mf{w}$ and $A\mf{w}$ in Definition \ref{norm_op}, respectively. 
\end{thm}
\begin{proof}
	From the assumption, it suffices to show that the first terms of both sides are equal.
	By the definition of the normalization operator, we have $T_{A\mf{w}}^{-1}A=T_\mf{w}^{-1}$.
	Therefore,
	\begin{align*}
	\max_{\Vert T_{A\mf{w}}^{-1}\bm{\epsilon} \Vert_p \leq \rho} L_S(A \mf{w}+\bm{\epsilon}) &=
	\max_{\Vert T_{A\mf{w}}^{-1}\bm{\epsilon} \Vert_p \leq \rho} L_S(\mf{w}+A^{-1}\bm{\epsilon})\\
	&= \max_{\Vert T_{A\mf{w}}^{-1}A\bm{\epsilon}' \Vert_p \leq \rho} L_S(\mf{w}+\bm{\epsilon}')\\
	&= \max_{\Vert T_\mf{w}^{-1}\bm{\epsilon}' \Vert_p \leq \rho} L_S(\mf{w}+\bm{\epsilon}')
	\end{align*}
	where $\bm{\epsilon}'=A^{-1}\bm\epsilon$.
\end{proof}

By Theorem~\ref{thm1}, adaptive sharpness defined in Equation \ref{eq:s4} is \textit{scale-invariant} as with training loss and generalization loss. 
This property makes the correlation of adaptive sharpness with the generalization gap stronger than that of sharpness in Equation~\ref{eq:s1}.

Figure~\ref{fig:cuboid} and \ref{fig:ellipsoid} show how a re-scaled weight vector can have the same adaptive sharpness value as that of the original weight vector. It can be observed that the boundary line of each region centered on $\mf{w}'$ is in contact with the red line. This implies that the maximum loss within each region centered on $\mf{w}'$ is maintained when $\Vert T^{-1}_{\mf{w}'} \bm{\epsilon}' \Vert_p \leq \rho $ is used for the maximization region. Thus, in this example, it can be seen that adaptive sharpness in \ref{fig:cuboid} and \ref{fig:ellipsoid} has scale-invariant property in contrast to sharpness of the spherical region shown in Figure~\ref{fig:sphere}.

The question that can be asked here is what kind of operators $T_\mf{w}$ can be considered as normalization operators which satisfy $T_{A\mf{w}}^{-1}A=T_\mf{w}^{-1}$ for any $A$ which does not change the loss function. One of the conditions for the scaling operator $A$ that does not change the loss function is that it should be node-wise scaling, which corresponds to row-wise or column-wise scaling in fully-connected layers and channel-wise scaling in convolutional layers. The effect of such node-wise scaling can be canceled using the inverses of the following operators:
\begin{itemize}
	\item element-wise
	\[T_\mf{w} = \mathrm{diag}(|w_1|, \ldots, |w_k|)\]
	where
    \[\mf{w} = [w_1, w_2, \ldots, w_k],\]
	\item filter-wise
	\begin{align}\label{filterwise}
	\nonumber T_\mf{w} = \mathrm{diag}(\mathrm{concat}(&\Vert \mf{f}_1\Vert _2 \mf{1}_{n(\mf{f}_1)}, \ldots, \Vert \mf{f}_m\Vert _2\mf{1}_{n(\mf{f}_m)},\\  &|w_1|, \ldots, |w_q|))
	\end{align}
	where
    \[\mf{w} = \mathrm{concat}(\mf{f}_1, \mf{f}_2, \ldots, \mf{f}_m, w_1, w_2, \ldots, w_q).\]
\end{itemize}

Here, $\mf{f}_i$ is the $i$-th flattened weight vector of a convolution filter and $w_j$ is the $j$-th weight parameter which is not included in any filters. And $m$ is the number of filters and $q$ is the number of other weight parameters in the model. If there is no convolutional layer in a model (i.e., $m=0$), then $q=k$ and both normalization operators are identical to each other.
Note that we use $T_\mf{w} + \eta I_k$ rather than $T_\mf{w}$ for sufficiently small $\eta > 0$ for stability.
$\eta$ is a hyper-parameter controlling trade-off between adaptivity and stability. 

\beginfigure
\centering
\ificml
\subfigure[Sharpness ($p=2$), \newline \hspace{1cm}  $\tau=0.174$.]{\includegraphics[width=0.48\linewidth]{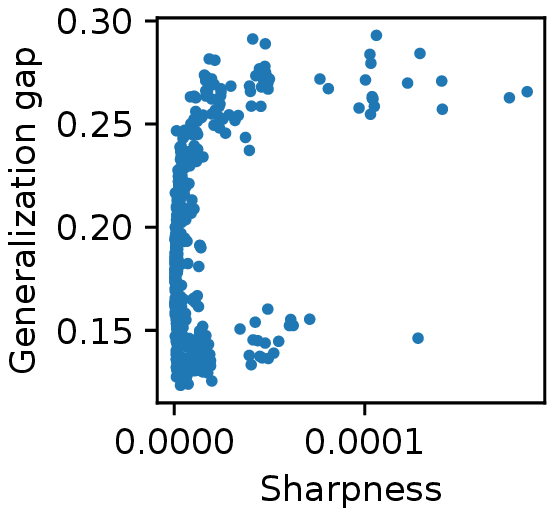}}
\subfigure[Adaptive sharpness ($p=2$), \newline $\tau=0.636$.]{\includegraphics[width=0.48\linewidth]{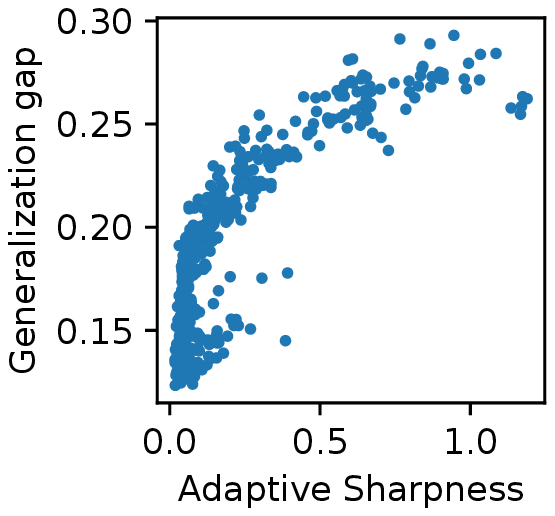}}
\subfigure[Sharpness ($p=\infty$), \newline $\tau=0.257$.]{\includegraphics[width=0.48\linewidth]{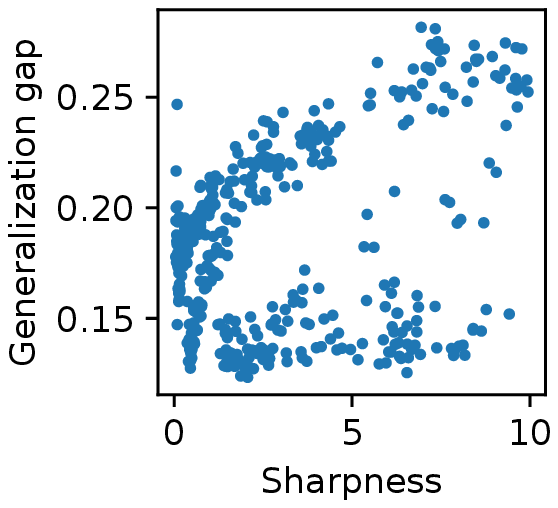}}
\subfigure[Adaptive sharpness \newline($p=\infty$), $\tau=0.616$.]{\includegraphics[width=0.48\linewidth]{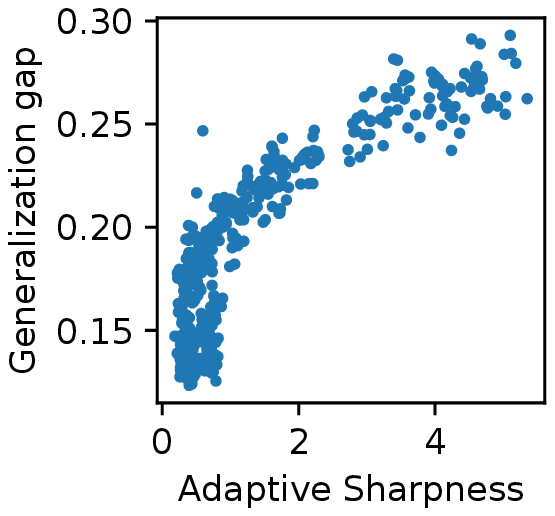}}
\else
\subfigure[Sharpness ($p=2$), $\tau=0.174$.]{\includegraphics[width=0.48\linewidth]{scatter_1.eps}}
\subfigure[Adaptive sharpness ($p=2$), $\tau=0.636$.]{\includegraphics[width=0.48\linewidth]{scatter_2.eps}}
\subfigure[Sharpness ($p=\infty$), $\tau=0.257$.]{\includegraphics[width=0.48\linewidth]{scatter_3.eps}}
\subfigure[Adaptive sharpness ($p=\infty$), $\tau=0.616$.]{\includegraphics[width=0.48\linewidth]{scatter_4.eps}}
\fi
\caption{Scatter plots which show correlation of sharpness and adaptive sharpness with respect to generalization gap and their rank correlation coefficients $\tau$.}\label{scatter}
\jmkendfigure

\begin{table}[t]
\ificml
\else
    \centering
    \captionsetup{justification=centering}
\fi
\setlength\tabcolsep{3.5pt}
\caption{Rank statistics for sharpness and adaptive sharpness. \label{rank}}
\begin{center}
\begin{small}
\begin{tabular}{lcccc}
\toprule
\multirow{2}{*}{} & \multicolumn{2}{c}{$p=2$} & \multicolumn{2}{c}{$p=\infty$}\\
\cmidrule{2-5}
& \multirow{1}{*}{sharpness} & \makecell{adaptive\\sharpness} & sharpness & \makecell{adaptive\\sharpness}\\
\midrule
$\tau$ (rank corr.) & $0.174$ & $\mf{0.636}$ & $0.257$ & $\mf{0.616}$\\
mini-batch size & $0.667$ & $\mf{0.696}$ & $0.777$ & $\mf{0.817}$\\
learning rate & $0.563$ & $\mf{0.577}$ & $0.797$ & $\mf{0.806}$\\
weight decay & $-0.297$ & $\mf{0.534}$ & $-0.469$ & $\mf{0.656}$\\
dropout rate & $\mf{0.102}$ & $-0.092$ & $0.161$ & $\mf{0.225}$\\
$\Psi$ (avg.) & $0.259$ & $\mf{0.429}$ & $0.316$ & $\mf{0.626}$\\
\bottomrule
\end{tabular}
\end{small}
\end{center}
\end{table}

To confirm that adaptive sharpness actually has a stronger correlation with generalization gap than sharpness, we compare \textit{rank statistics} which demonstrate the change of adaptive sharpness and sharpness with respect to generalization gap. For correlation analysis, we use $4$ hyper-parameters: mini-batch size, initial learning rate, weight decay coefficient and dropout rate. As can be seen in Table~\ref{rank}, Kendall rank correlation coefficient \citep{kendall} of adaptive sharpness is greater than that of sharpness regardless of the value of $p$. Furthermore, we compare granulated coefficients \citep{jiang2019fantastic} with respect to different hyper-parameters to measure the effect of each hyper-parameter separately. In Table~\ref{rank}, the coefficients of adaptive sharpness are higher in most hyper-parameters and the average as well. Scatter plots illustrated in Figure~\ref{scatter} also show stronger correlation of adaptive sharpness. The difference in correlation behaviors of adaptive sharpness and sharpness provides an evidence that scale-invariant property helps strengthen the correlation with generalization gap. The experimental details are described in Appendix~\ref{corrdetail}.

Although there are various normalization methods other than the normalization operators introduced above, this paper covers only element-wise and filter-wise normalization operators. Node-wise normalization can also be viewed as a normalization operator. \citet{tsuzuku2020normalized} suggest node-wise normalization method for obtaining normalized flatness. However, the method requires that the parameter should be at a critical point. Also, in the case of node-wise normalization using unit-invariant SVD \citep{uhlmann2018generalized}, there is a concern that the speed of the optimizer can be degraded due to the significant additional cost for scale-direction decomposition of weight tensors. Therefore the node-wise normalization is not covered in this paper. In the case of layer-wise normalization using spectral norm or Frobenius norm of weight matrices \citep{neyshabur2017exploring}, the condition $T_{A\mf{w}}^{-1}A=T_\mf{w}^{-1}$ is not satisfied. Therefore, it cannot be used for adaptive sharpness so we do not cover it in this paper. 

Meanwhile, even though all weight parameters including biases can have scale-dependency, there remains more to consider when applying normalization to the biases. In terms of bias parameters of rectifier neural networks, there also exists translation-dependency in sharpness, which weakens the correlation with the generalization gap as well. Using the similar arguments as in the proof of Theorem~\ref{thm1}, it can be derived that diagonal elements of $T_{\mf{w}}$ corresponding to biases must be replaced by constants to guarantee translation-invariance, which induces adaptive sharpness that corresponds to the case of not applying bias normalization. We compare the generalization performance based on adaptive sharpness with and without bias normalization, in Section~\ref{sec:e}.

There are several previous works which are closely related to adaptive sharpness.
\citet{li2018visualizing}, which suggest a methodology for visualizing loss landscape, is related to adaptive sharpness. In that study, filter-wise normalization which is equivalent to the definition in Equation~\ref{filterwise} is used to remove scale-dependency from loss landscape and make comparisons between loss functions meaningful. In spite of their empirical success, \citet{li2018visualizing} do not provide a theoretical evidence for explaining how filter-wise scaling contributes the scale-invariance and correlation with generalization. In this paper, we clarify how the filter-wise normalization relates to generalization by proving the scale-invariant property of adaptive sharpness in Theorem \ref{thm1}.

Also, sharpness suggested in \citet{keskar2017large} can be regarded as a special case of adaptive sharpness which uses $p=\infty$ and the element-wise normalization operator.
\citet{jiang2019fantastic} confirm experimentally that the adaptive sharpness suggested in \citet{keskar2017large} shows a higher correlation with the generalization gap than sharpness which does not use element-wise normalization operator.
This experimental result implies that Theorem~\ref{thm1} is also practically validated.

Therefore, it seems that sharpness with $p=\infty$ suggested by \citet{keskar2017large} also can be used directly for learning as it is, but a problem arises in terms of generalization performance in learning. \citet{foret2021sharpnessaware} confirm experimentally that the generalization performance with sharpness defined in square region $\Vert \bm{\epsilon}\Vert _{\infty} \leq \rho$ result is worse than when SAM is performed with sharpness defined in spherical region $\Vert \bm{\epsilon}\Vert _{2} \leq \rho$.

We conduct performance comparison tests for $p=2$ and $p=\infty$, and experimentally reveal that $p=2$ is more suitable for learning as in \citet{foret2021sharpnessaware}. The experimental results are shown in Section~\ref{sec:e}.

\section{Adaptive Sharpness-Aware Minimization} \label{sec:a}

In the previous section, we introduce a scale-invariant measure called adaptive sharpness to overcome the limitation of sharpness.
As in sharpness, we can obtain a generalization bound using adaptive sharpness, which is presented in the following theorem.
\begin{thm}\label{thm2}
Let $T^{-1}_\mf{w}$ be the normalization operator on $\mathbb{R}^k$.
If $L_D(\mf w) \leq E_{\epsilon_i \sim \mathcal{N}(0, \sigma^2)} [L_D(\mf w+\bm\epsilon)]$ for some $\sigma > 0$,
then with probability $1-\delta$, 
\begin{equation} \label{eq:a1}
L_{D}(\mf{w}) \leq \max_{\Vert T^{-1}_\mf{w} \bm{\epsilon}\Vert _{2} \leq \rho } L_{S}(\mf{w}+\bm{\epsilon}) + h\left(\frac{\Vert \mf{w}\Vert _{2}^2}{\eta^2 \rho^2}\right)
\end{equation}
where $h: \mathbb{R}^+ \rightarrow \mathbb{R}^+$ is a strictly increasing function, $n = \vert S\vert$ and $\rho = \sqrt{k}\sigma (1+\sqrt{\log n / k}) /\eta$.
\end{thm}

\begin{algorithm}[tb]
   \caption{ASAM algorithm ($p=2$)}
   \label{alg1}
\begin{algorithmic}
   \STATE {\bfseries Input:} Loss function $l$, training dataset $S:=\cup^n_{i=1}\{(\mf{x}_i, \mf{y}_i)\}$, mini-batch size $b$, radius of maximization region $\rho$, weight decay coefficient $\lambda$, scheduled learning rate $\alpha$, initial weight $\mf{w}_0$.
   \STATE {\bfseries Output:} Trained weight $\mf{w}$
   \STATE Initialize weight $\mf{w} \coloneqq \mf{w}_0$
   \STATE {\bfseries while} not converged  {\bfseries do} 
   \STATE   \hspace{0.05\linewidth} Sample a mini-batch $B$ of size $b$ from $S$
   \STATE   \hspace{0.05\linewidth} $\bm{\epsilon} \coloneqq \rho \frac {\displaystyle T^2_{\mf{w}} \nabla L_B(\mf{w})} {\displaystyle \Vert T_{\mf{w}} \nabla L_B(\mf{w}) \Vert_2}$
   \STATE   \hspace{0.05\linewidth} $\mf{w} \coloneqq \mf{w} - \alpha \left( \nabla L_B(\mf{w} + \bm{\epsilon}) + \lambda \mf{w} \right)^{\dagger}$
   \STATE {\bfseries end while}
   \STATE {\bfseries return} $\mf{w}$
\end{algorithmic}
\end{algorithm}

Note that Theorem \ref{thm2} still holds for $p > 2$ due to the monotonicity of $p$-norm, i.e.,
if $0<r<p$, $\Vert\mf{x}\Vert_p \leq \Vert\mf{x}\Vert_r$ for any $\mf{x}\in\mathbb{R}^n$.
If $T_\mf{w}$ is an identity operator, Equation~\ref{eq:a1} is reduced equivalently to Equation~\ref{eq:p1}.
The proof of Equation~\ref{eq:a1} is described in detail in Appendix~\ref{app1}. 

The right hand side of Equation \ref{eq:a1}, i.e., generalization bound, can be expressed using adaptive sharpness as
\ificml
    \begin{equation*}
    \resizebox{\hsize}{!}{$\displaystyle\left(\max_{\Vert T^{-1}_\mf{w} \bm{\epsilon}\Vert _{p} \leq \rho } L_{S}(\mf{w}+\bm{\epsilon}) - L_S(\mf{w})\right) + L_S(\mf{w}) + h\left(\frac{\Vert \mf{w}\Vert _{2}^2}{\eta^2 \rho^2}\right).$}
    \end{equation*}
\else
    \begin{equation*}
    \left(\max_{\Vert T^{-1}_\mf{w} \bm{\epsilon}\Vert _{p} \leq \rho } L_{S}(\mf{w}+\bm{\epsilon}) - L_S(\mf{w})\right) + L_S(\mf{w}) + h\left(\frac{\Vert \mf{w}\Vert _{2}^2}{\eta^2 \rho^2}\right).
    \end{equation*}
\fi

Since $h\left(\Vert \mf{w}\Vert _{2}^2 / \eta^2 \rho^2 \right)$ is a strictly increasing function with respect to $\Vert \mf{w}\Vert _2^2$, it can be substituted with $\ell^2$ weight decaying regularizer. Therefore, we can define adaptive sharpness-aware minimization problem as
\begin{equation}\label{eq:a4}
    \min_{\mf{w}} \max_{\Vert T^{-1}_\mf{w} \bm{\epsilon}\Vert_{p} \leq \rho } L_{S}(\mf{w}+\bm{\epsilon}) + \frac{\lambda}{2}\Vert \mf{w}\Vert _2^2.
\end{equation}
To solve the minimax problem in Equation \ref{eq:a4}, it is necessary to find optimal $\bm{\epsilon}$ first.
Analogous to SAM, we can approximate the optimal $\bm{\epsilon}$ to maximize $L_{S}(\mf{w}+\bm{\epsilon})$ using a first-order approximation as
\begin{align*}
    \tilde{\bm{\epsilon}}_{t} 
    & = \argmax_{\Vert \tilde{\bm{\epsilon}} \Vert _{p} \leq \rho } L_{S}(\mf{w}_t + T_{\mf{w}_t} \tilde{\bm{\epsilon}}) \\
    & \approx \argmax_{\Vert \tilde{\bm{\epsilon}}\Vert _{p} \leq \rho } \tilde{\bm{\epsilon}}^\top T_{\mf{w}_t} \nabla L_S(\mf{w}_t) \\
    & = \rho \operatorname{sign}(\nabla L_S(\mf{w}_t)) \frac {\displaystyle \vert T_{\mf{w}_t} \nabla L_S(\mf{w}_t) \vert^{q-1}} {\displaystyle \Vert T_{\mf{w}_t} \nabla L_S(\mf{w}_t) \Vert^{q-1}_q}
\end{align*}
\noindent where $\tilde{\bm{\epsilon}} = T^{-1}_{\mf{w}} \bm{\epsilon}.$
Then, the two-step procedure for adaptive sharpness-aware minimization (ASAM) is expressed as
\begin{equation*}
  \left\{
    \begin{array}{ll}
      \bm{\epsilon}_{t} = \rho T_{\mf{w}_t}\operatorname{sign}(\nabla L_S(\mf{w}_t)) \frac {\displaystyle \vert T_{\mf{w}_t} \nabla L_S(\mf{w}_t) \vert^{q-1}} {\displaystyle \Vert T_{\mf{w}_t} \nabla L_S(\mf{w}_t) \Vert^{q-1}_q} \\
      \mf{w}_{t+1} = \mf{w}_t - \alpha_t \left( \nabla L_S(\mf{w}_t + \bm{\epsilon}_{t}) + \lambda \mf{w}_t \right)
    \end{array}
  \right.
\end{equation*}
for $t=0,1,2,\cdots$.
Especially, if $p=2$,
\[\bm{\epsilon}_t = \rho \frac {\displaystyle T^2_{\mf{w}_t} \nabla L_S(\mf{w}_t)} {\displaystyle \Vert T_{\mf{w}_t} \nabla L_S(\mf{w}_t) \Vert_2}\]
and if $p=\infty$,
\[\bm{\epsilon}_t = \rho T_{\mf{w}_t} \operatorname{sign}(\nabla L_S(\mf{w}_t)).\]

In this study, experiments are conducted on ASAM in cases of $p=\infty$ and $p=2$. 
The ASAM algorithm with $p=2$ is described in detail on Algorithm~\ref{alg1}. Note that the SGD \citep{nesterov1983method} update marked with $\dagger$ in Algorithm~\ref{alg1} can be combined with momentum or be replaced by update of another optimization scheme such as Adam~\citep{kingma2015adam}.

\section{Experimental Results} \label{sec:e}
In this section, we evaluate the performance of ASAM.
We first show how SAM and ASAM operate differently in a toy example.
We then compare the generalization performance of ASAM with other learning algorithms for various model architectures and various datasets: CIFAR-10, CIFAR-100 \citep{krizhevsky2009cifar}, ImageNet \citep{imagenet} and IWSLT'14 DE-EN~\citep{cettolo2014report}. Finally, we show how robust to label noise ASAM is.

\subsection{Toy Example} \label{toy}

\begin{figure}
\centering
\captionsetup{justification=centering}
\includegraphics[width=\linewidth]{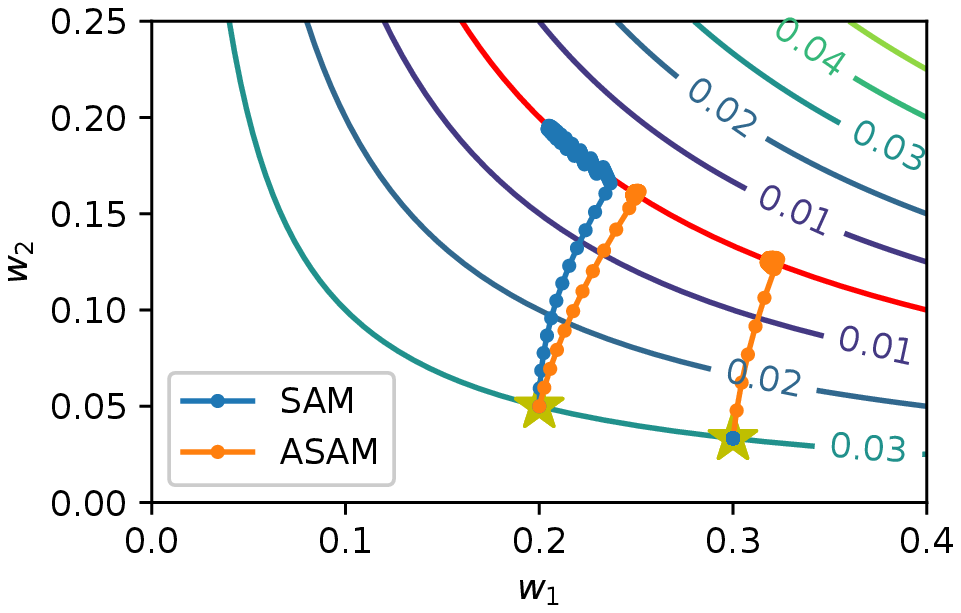}
\caption{Trajectories of SAM and ASAM.}\label{contour}
\end{figure}

As mentioned in Section \ref{sec:a}, sharpness varies by parameter re-scaling even if its loss function remains the same, while adaptive sharpness does not.
To elaborate this, we consider a simple loss function $L(\mf{w})=|w_1\mathrm{ReLU}(w_2) - 0.04|$ where $\mf{w}=(w_1,w_2)\in\mathbb{R}^2$.
Figure \ref{contour} presents the trajectories of SAM and ASAM with two different initial weights $\mf{w}_0=(0.2,0.05)$ and $\mf{w}_0=(0.3,0.033)$.
The red line represents the set of minimizers of the loss function $L$, i.e., $\{(w_1,w_2); w_1w_2=0.04,~w_2>0\}$.
As seen in Figure \ref{fig:sphere}, sharpness is maximized when $w_1=w_2$ within the same loss contour line, and therefore SAM tries to converge to $(0.2,0.2)$.
Here, we use $\rho=0.05$ as in \citet{foret2021sharpnessaware}.
On the other hand, adaptive sharpness remains the same along the same contour line which implies that ASAM converges to the point in the red line near the initial point as can be seen in Figure \ref{contour}.

Since SAM uses a fixed radius in a spherical region for minimizing sharpness, it may cause undesirable results depending on the loss surface and the current weight.
If $\mf{w}_0=(0.3,0.033)$, while SAM even fails to converge to the valley with $\rho=0.05$, ASAM converges no matter which $\mf{w}_0$ is used if $\rho<\sqrt{2}$. In other words, appropriate $\rho$ for SAM is dependent on the scales of $\mf{w}$ on the training trajectory, whereas $\rho$ of ASAM is not.

\begin{figure}[!t]
\begin{center}
\centerline{\includegraphics[width=\columnwidth]{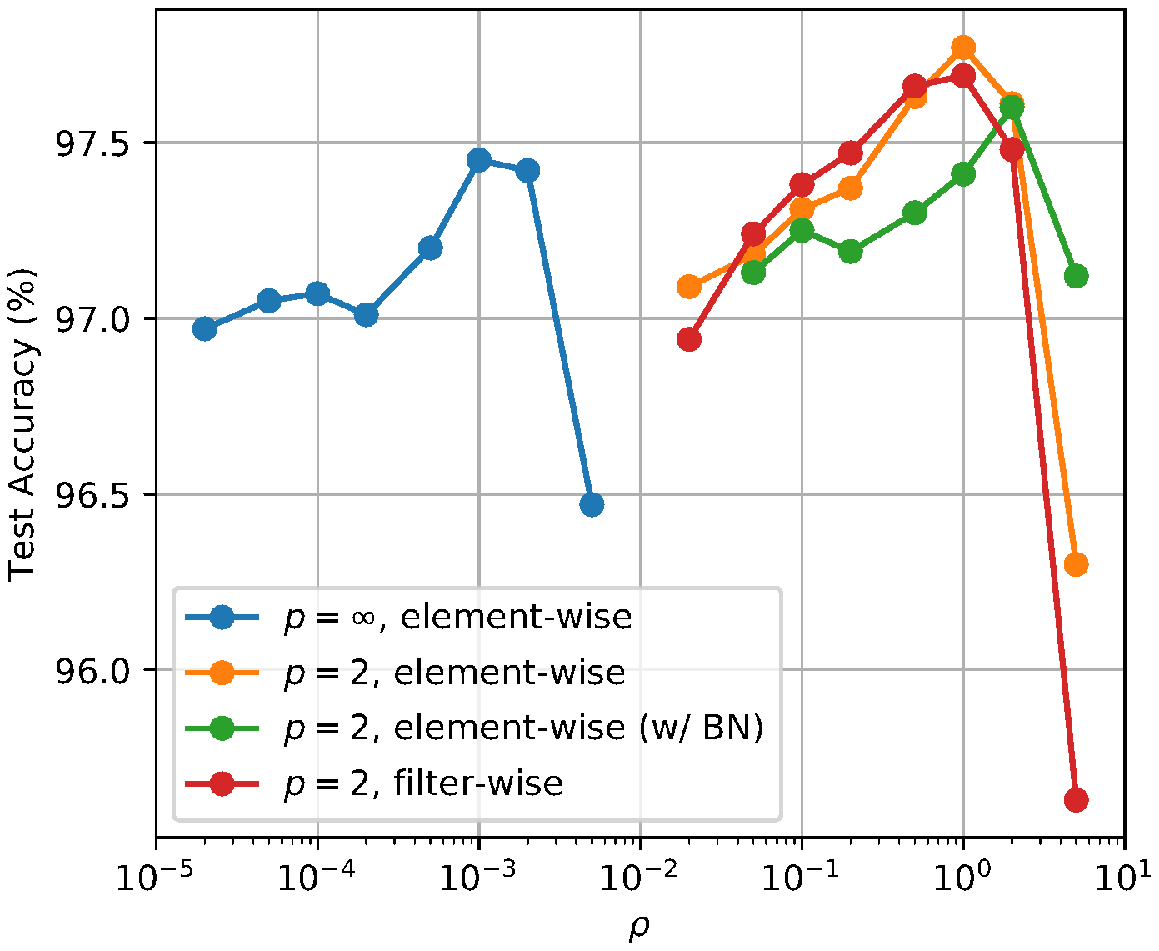}}
\caption{
Test accuracy curves obtained from ASAM algorithm using a range of $\rho$ with different factors: element-wise normalization with $p=\infty$, element-wise normalization with $p=2$ with and without bias normalization (BN) and filter-wise normalization with $p=2$.
}
\label{normalization}
\end{center}
\end{figure}

\begin{table}[t]
\ificml
\else
    \centering
\fi
\begin{threeparttable}
\caption{Maximum test accuracies for SGD, SAM and ASAM on CIFAR-10 dataset. \label{table1}}
\begin{center}
\begin{small}
\begin{tabular}{lcccr}
\toprule
Model & SGD & SAM & ASAM \\
\midrule
DenseNet-121 & $91.00_{\pm 0.13}$ & $92.00_{\pm 0.17}$ & $\mf{93.33}_{\pm 0.04}$ \\
ResNet-20 & $93.18_{\pm 0.21}$ & $93.56_{\pm 0.15}$ & $\mf{93.82}_{\pm 0.17}$ \\
ResNet-56 & $94.58_{\pm 0.20}$ & $95.18_{\pm 0.15}$ & $\mf{95.42}_{\pm 0.16}$ \\
VGG19-BN$^*$ & $93.87_{\pm 0.09}$ & $94.60$ & $\mf{95.07}_{\pm 0.05}$ \\
{\fontsize{8}{9.6}\selectfont ResNeXt29-32x4d} & $95.84_{\pm 0.24}$ & $96.34_{\pm 0.30}$ & $\mf{96.80}_{\pm 0.06}$ \\
WRN-28-2 & $95.13_{\pm 0.16}$ & $95.74_{\pm 0.08}$ & $\mf{95.94}_{\pm 0.05}$ \\
WRN-28-10 & $96.34_{\pm 0.12}$ & $96.98_{\pm 0.04}$ & $\mf{97.28}_{\pm 0.07}$ \\
\midrule
{\fontsize{8.5}{9.6}\selectfont PyramidNet-272}$^\dagger$ & $98.44_{\pm 0.08}$ & $98.55_{\pm 0.05}$ & $\mf{98.68}_{\pm 0.08}$ \\
\bottomrule
\end{tabular}
\begin{tablenotes}
\item[*] Some runs completely failed, thus giving 10\% of accuracy (success rate: SGD: 3/5, SAM 1/5, ASAM 3/5)
\item[$\dagger$] PyramidNet-272 architecture is tested 3 times for each learning algorithm.
\end{tablenotes}
\end{small}
\end{center}
\end{threeparttable}
\end{table}

\begin{table}[t]
\ificml
\else
    \centering
\fi
\begin{threeparttable}
\caption{Maximum test accuracies for SGD, SAM and ASAM on CIFAR-100 dataset. \label{table2}}
\begin{center}
\begin{small}
\begin{tabular}{lcccr}
\toprule
Model & SGD & SAM & ASAM \\
\midrule
DenseNet-121 & $68.70_{\pm 0.31}$ & $69.84_{\pm 0.12}$ & $\mf{70.60}_{\pm 0.20}$ \\
ResNet-20 & $69.76_{\pm 0.44}$ & $71.06_{\pm 0.31}$ & $\mf{71.40}_{\pm 0.30}$ \\
ResNet-56 & $73.12_{\pm 0.19}$ & $75.16_{\pm 0.05}$ & $\mf{75.86}_{\pm 0.22}$ \\
VGG19-BN$^*$ & $71.80_{\pm 1.35}$ & $73.52_{\pm 1.74}$ & $\mf{75.80}_{\pm 0.27}$ \\
{\fontsize{8}{9.6}\selectfont ResNeXt29-32x4d} & $79.76_{\pm 0.23}$ & $81.48_{\pm 0.17}$ & $\mf{82.30}_{\pm 0.11}$ \\
WRN-28-2 & $75.28_{\pm 0.17}$ & $77.25_{\pm 0.35}$ & $\mf{77.54}_{\pm 0.14}$ \\
WRN-28-10 & $81.56_{\pm 0.13}$ & $83.42_{\pm 0.04}$ & $\mf{83.68}_{\pm 0.12}$ \\
\midrule
{\fontsize{8.5}{9.6}\selectfont PyramidNet-272}$^\dagger$ & $88.91_{\pm 0.12}$ & $89.36_{\pm 0.20}$ & $\mf{89.90}_{\pm 0.13}$ \\
\bottomrule
\end{tabular}
\begin{tablenotes}
\item[*] Some runs completely failed, thus giving 10\% of accuracy (success rate: SGD: 5/5, SAM 4/5, ASAM 4/5)
\item[$\dagger$] PyramidNet-272 architecture is tested 3 times for each learning algorithm.
\end{tablenotes}
\end{small}
\end{center}
\end{threeparttable}
\end{table}

\subsection{Image Classification: CIFAR-10/100 and ImageNet} \label{cifar}
To confirm the effectiveness of ASAM, we conduct comparison experiments with SAM using CIFAR-10 and CIFAR-100 datasets.
We use the same data split as the original paper \citep{krizhevsky2009cifar}. The hyper-parameters used in this test is described in Table~\ref{table2}. Before the comparison tests with SAM, there are three factors to be chosen in ASAM algorithm:
\begin{itemize}
\item normalization schemes: element-wise vs. filter-wise
\item $p$-norm: $p=\infty$ vs. $p=2$
\item bias normalization: with vs. without
\end{itemize}

First, we perform the comparison test for filter-wise and element-wise normalization using WideResNet-16-8 model \citep{zagoruyko2016wide} and illustrate the results in Figure~\ref{normalization}. As can be seen, both test accuracies are comparable across $\rho$, and element-wise normalization provides a slightly better accuracy at $\rho=1.0$. 

Similarly, Figure~\ref{normalization} shows how much test accuracy varies with the maximization region for adaptive sharpness. It can be seen that $p=2$ shows better test accuracies than $p=\infty$, which is consistent with \citet{foret2021sharpnessaware}. 
We could also observe that bias normalization does not contribute to the improvement of test accuracy in Figure~\ref{normalization}. Therefore, we decide to use element-wise normalization operator and $p=2$, and not to employ bias normalization in the remaining tests.

As ASAM has a hyper-parameter $\rho$ to be tuned, we first conduct a grid search over \{$0.00005$, $0.0001$, $0.0002$, \ldots, $0.5$, $1.0$, $2.0$\} for finding appropriate values of $\rho$. We use $\rho=0.5$ for CIFAR-10 and $\rho=1.0$ for CIFAR-100, because it gives moderately good performance across various models. We set $\rho$ for SAM as $0.05$ for CIFAR-10 and $0.1$ for CIFAR-100 as in \citet{foret2021sharpnessaware}. $\eta$ for ASAM is set to $0.01$. 
We set mini-batch size to $128$, and $m$-sharpness suggested by \citet{foret2021sharpnessaware} is not employed. The number of epochs is set to $200$ for SAM and ASAM and $400$ for SGD. Momentum and weight decay coefficient are set to $0.9$ and $0.0005$, respectively. 
Cosine learning rate decay \citep{loshchilov2016sgdr} is adopted with an initial learning rate $0.1$. Also, random resize, padding by four pixels, normalization and random horizontal flip are applied for data augmentation and label smoothing \citep{NEURIPS2019_f1748d6b} is adopted with its factor of $0.1$.

Using the hyper-parameters, we compare the best test accuracies obtained by SGD, SAM and ASAM for various rectifier neural network models: VGG \citep{simonyan2014very}, ResNet \citep{he2016deep}, DenseNet \citep{huang2017densely}, WideResNet \citep{zagoruyko2016wide}, and ResNeXt \citep{xie2017aggregated}.

For PyramidNet-272~\citep{han2017deep}, we additionally apply some latest techniques: AutoAugment~\citep{cubuk2019autoaugment}, CutMix~\citep{yun2019cutmix} and ShakeDrop~\citep{yamada2019shakedrop}. We employ the $m$-sharpness strategy with $m=32$. Initial learning rate and mini-batch size are set to $0.05$ and $256$, respectively. The number of epochs is set to $900$ for SAM and ASAM and $1800$ for SGD. We choose $\rho$ for SAM as $0.05$, as in \citet{foret2021sharpnessaware}, and $\rho$ for ASAM as $1.0$ for both CIFAR-10 and CIFAR-100. Every entry in the tables represents mean and standard deviation of 5 independent runs. In both CIFAR-10 and CIFAR-100 cases, ASAM generally surpasses SGD and SAM, as can be seen in Table~\ref{table1} and Table~\ref{table2}.

For the sake of evaluations at larger scale, we compare the performance of SGD, SAM and ASAM on ImageNet. We apply each method with ResNet-50 and use $\rho=0.05$ for SAM and $\rho=1.0$ for ASAM. The number of training epochs is $200$ for SGD and $100$ for SAM and ASAM. We use mini-batch size $512$, initial learning rate $0.2$, and SGD optimizer with weight decay coefficient $0.0001$. Other hyper-parameters are the same as those of CIFAR-10/100 tests. We also employ $m$-sharpness with $m=128$ for both SAM and ASAM.

Table~\ref{imagenet} shows mean and standard deviation of maximum test accuracies over $3$ independent runs for each method. As can be seen in the table, ASAM achieves higher accuracies than SGD and SAM. These results imply that ASAM can enhance generalization performance of rectifier neural network architectures in image classification task beyond CIFAR.

\begin{table}[h]
\setlength\tabcolsep{4.5pt}
\caption{Top1 and Top5 maximum test accuracies for SGD, SAM and ASAM on ImageNet dataset using ResNet-50. \label{imagenet}}
\begin{center}
\begin{small}
\begin{tabular}{lcccc}
\toprule
& SGD & SAM & ASAM\\
\midrule
Top1 & $75.79_{\pm 0.22}$ & $76.39_{\pm 0.03}$ & $\mf{76.63}_{\pm 0.18}$\\
Top5 & $92.62_{\pm 0.04}$ & $92.97_{\pm 0.07}$ & $\mf{93.16}_{\pm 0.18}$\\
\bottomrule
\end{tabular}
\end{small}
\end{center}
\end{table}

\subsection{Machine Translation: IWSLT'14 DE-EN}
To validate effectiveness of ASAM in tasks other than image classification, we apply SAM and ASAM to IWSLT'14 DE-EN, a dataset on machine translation task. 

In this test, we adopt Transformer architecture~\citep{vaswani2017attention} and Adam optimizer as a base optimizer of SAM and ASAM instead of SGD. Learning rate, $\beta_1$ and $\beta_2$ for Adam are set to $0.0005$, $0.9$ and $0.98$, respectively. Dropout rate and weight decay coefficient are set to 0.3 and 0.0001, respectively. Label smoothing is adopted with its factor $0.1$. We choose $\rho=0.1$ for SAM and $\rho=0.2$ for ASAM as a result of a grid search over \{$0.005$, $0.01$, $0.02$, \ldots, $0.5$, $1.0$, $2.0$\} using validation dataset. The results of the experiments are obtained from 3 independent runs.

As can be seen in Table~\ref{iwslt}, we could observe improvement even on IWSLT'14 in BLEU score when using Adam+ASAM instead of Adam or Adam+SAM.

\begin{table}[h]
\setlength\tabcolsep{4.5pt}
\caption{BLEU scores for Adam, Adam+SAM and Adam+ASAM on IWSLT'14 DE-EN dataset using Transformer. \label{iwslt}}
\begin{center}
\begin{small}
\begin{tabular}{lccc}
\toprule
BLEU score & Adam & Adam+SAM & Adam+ASAM\\
\midrule
Validation & $35.34_{\pm <0.01}$ & $35.52_{\pm 0.01}$ & $\mathbf{35.66}_{\pm <0.01}$\\
Test & $34.86_{\pm <0.01}$ & $34.78_{\pm 0.01}$ & $\mathbf{35.02}_{\pm <0.01}$\\
\bottomrule
\end{tabular}
\end{small}
\end{center}
\end{table}


\subsection{Robustness to Label Noise} \label{robust}

As shown in \citet{foret2021sharpnessaware}, SAM is as robust to label noise in the training data as MentorMix \citep{jiang2020beyond}, which is a state-of-the-art method.  
We expect that ASAM would share the robustness to label noise. To confirm this, we compare the test accuracies of SGD, SAM and ASAM for ResNet-32 model and CIFAR-10 dataset whose labels in the training data are corrupted by symmetric label noise \citep{van2015learning} with noise levels of 20\%, 40\%, 60\% and 80\%, and the test data is not touched. Hyper-parameter settings are the same as that of previous CIFAR experiments.
Table~\ref{noise32} shows test accuracies for SGD, SAM and ASAM obtained from 3 independent runs with respect to label noise levels. Compared to SGD and SAM, ASAM generally enhances the test accuracy across various noise level by retaining the robustness to label noise.

\begintable
\caption{Maximum test accuracies of ResNet-32 models trained on CIFAR-10 with label noise.\label{noise32}}
\begin{center}
\begin{small}
\begin{tabular}{lccccr}
\toprule
Noise rate & SGD & SAM & ASAM \\
\midrule
0\% & $94.50_{\pm 0.11}$ & $94.80_{\pm 0.12}$ & $\mf{94.88}_{\pm 0.12}$ \\
20\% & $91.32_{\pm 0.23}$ & $92.94_{\pm 0.12}$ & $\mf{93.21}_{\pm 0.10}$ \\
40\% & $87.68_{\pm 0.05}$ & $90.62_{\pm 0.18}$ & $\mf{90.89}_{\pm 0.13}$ \\
60\% & $82.50_{\pm 0.30}$ & $86.58_{\pm 0.30}$ & $\mf{87.41}_{\pm 0.16}$ \\
80\% & $68.35_{\pm 0.85}$ & $\mf{69.92}_{\pm 0.98}$ & $67.69_{\pm 1.34}$ \\
\bottomrule
\end{tabular}
\end{small}
\end{center}
\jmkendtable

\ificml
\else
    \clearpage
\fi

\section{Conclusions}\label{sec:c}
In this paper, we have introduced adaptive sharpness with scale-invariant property that improves training path in weight space by adjusting their maximization region with respect to weight scale. Also, we have confirmed that this property, which ASAM shares, contributes in improvement of generalization performance. The superior performance of ASAM is notable from the comparison tests conducted against SAM, which is currently state-of-the-art learning algorithm in many image classification benchmarks. In addition to the contribution as a learning algorithm, adaptive sharpness can serve as a generalization measure with stronger correlation with generalization gap benefiting from their scale-invariant property. Therefore adaptive sharpness has a potential to be a metric for assessment of neural networks. We have also suggested the condition of normalization operator for adaptive sharpness but we did not cover all the normalization schemes which satisfy the condition. So this area could be further investigated for better generalization performance in future works. 

\section{Acknowledgements}
We would like to thank Kangwook Lee, Jaedeok Kim and Yonghyun Ryu for supports on our machine translation experiments. We also thank our other colleagues at Samsung Research - Joohyung Lee, Chiyoun Park and Hyun-Joo Jung - for their insightful discussions and feedback.

\ificml
\else
    \clearpage
\fi

\bibliography{example_paper}

\begin{thebibliography}{40}
\providecommand{\natexlab}[1]{#1}
\providecommand{\url}[1]{\texttt{#1}}
\expandafter\ifx\csname urlstyle\endcsname\relax
  \providecommand{\doi}[1]{doi: #1}\else
  \providecommand{\doi}{doi: \begingroup \urlstyle{rm}\Url}\fi

\bibitem[Cettolo et~al.(2014)Cettolo, Niehues, St{\"u}ker, Bentivogli, and
  Federico]{cettolo2014report}
Cettolo, M., Niehues, J., St{\"u}ker, S., Bentivogli, L., and Federico, M.
\newblock Report on the 11th {IWSLT} evaluation campaign, {IWSLT} 2014.
\newblock In \emph{Proceedings of the International Workshop on Spoken Language
  Translation, Hanoi, Vietnam}, volume~57, 2014.

\bibitem[Chatterji et~al.(2019)Chatterji, Neyshabur, and
  Sedghi]{chatterji2019intriguing}
Chatterji, N., Neyshabur, B., and Sedghi, H.
\newblock The intriguing role of module criticality in the generalization of
  deep networks.
\newblock In \emph{International Conference on Learning Representations}, 2019.

\bibitem[Chaudhari et~al.(2019)Chaudhari, Choromanska, Soatto, LeCun, Baldassi,
  Borgs, Chayes, Sagun, and Zecchina]{chaudhari2019entropy}
Chaudhari, P., Choromanska, A., Soatto, S., LeCun, Y., Baldassi, C., Borgs, C.,
  Chayes, J., Sagun, L., and Zecchina, R.
\newblock Entropy-sgd: Biasing gradient descent into wide valleys.
\newblock \emph{Journal of Statistical Mechanics: Theory and Experiment},
  2019\penalty0 (12):\penalty0 124018, 2019.

\bibitem[Cubuk et~al.(2019)Cubuk, Zoph, Mane, Vasudevan, and
  Le]{cubuk2019autoaugment}
Cubuk, E.~D., Zoph, B., Mane, D., Vasudevan, V., and Le, Q.~V.
\newblock Autoaugment: Learning augmentation strategies from data.
\newblock In \emph{Proceedings of the IEEE/CVF Conference on Computer Vision
  and Pattern Recognition}, pp.\  113--123, 2019.

\bibitem[Deng et~al.(2009)Deng, Dong, Socher, Li, Li, and Fei-Fei]{imagenet}
Deng, J., Dong, W., Socher, R., Li, L.-J., Li, K., and Fei-Fei, L.
\newblock Image{N}et: A large-scale hierarchical image database.
\newblock In \emph{2009 IEEE conference on computer vision and pattern
  recognition}, pp.\  248--255. IEEE, 2009.

\bibitem[Dinh et~al.(2017)Dinh, Pascanu, Bengio, and Bengio]{dinh2017sharp}
Dinh, L., Pascanu, R., Bengio, S., and Bengio, Y.
\newblock Sharp minima can generalize for deep nets.
\newblock In \emph{International Conference on Machine Learning}, pp.\
  1019--1028. PMLR, 2017.

\bibitem[Foret et~al.(2021)Foret, Kleiner, Mobahi, and
  Neyshabur]{foret2021sharpnessaware}
Foret, P., Kleiner, A., Mobahi, H., and Neyshabur, B.
\newblock Sharpness-aware minimization for efficiently improving
  generalization.
\newblock In \emph{International Conference on Learning Representations}, 2021.

\bibitem[Han et~al.(2017)Han, Kim, and Kim]{han2017deep}
Han, D., Kim, J., and Kim, J.
\newblock Deep pyramidal residual networks.
\newblock In \emph{Proceedings of the IEEE conference on computer vision and
  pattern recognition}, pp.\  5927--5935, 2017.

\bibitem[He et~al.(2016)He, Zhang, Ren, and Sun]{he2016deep}
He, K., Zhang, X., Ren, S., and Sun, J.
\newblock Deep residual learning for image recognition.
\newblock In \emph{Proceedings of the IEEE conference on computer vision and
  pattern recognition}, pp.\  770--778, 2016.

\bibitem[Hochreiter \& Schmidhuber(1997)Hochreiter and
  Schmidhuber]{hochreiter1997flat}
Hochreiter, S. and Schmidhuber, J.
\newblock Flat minima.
\newblock \emph{Neural computation}, 9\penalty0 (1):\penalty0 1--42, 1997.

\bibitem[Hochreiter et~al.(1995)Hochreiter, Schmidhuber,
  et~al.]{hochreiter1995simplifying}
Hochreiter, S., Schmidhuber, J., et~al.
\newblock Simplifying neural nets by discovering flat minima.
\newblock \emph{Advances in neural information processing systems}, pp.\
  529--536, 1995.

\bibitem[Huang et~al.(2017)Huang, Liu, Van Der~Maaten, and
  Weinberger]{huang2017densely}
Huang, G., Liu, Z., Van Der~Maaten, L., and Weinberger, K.~Q.
\newblock Densely connected convolutional networks.
\newblock In \emph{Proceedings of the IEEE conference on computer vision and
  pattern recognition}, pp.\  4700--4708, 2017.

\bibitem[Jiang et~al.(2020)Jiang, Huang, Liu, and Yang]{jiang2020beyond}
Jiang, L., Huang, D., Liu, M., and Yang, W.
\newblock Beyond synthetic noise: Deep learning on controlled noisy labels.
\newblock In \emph{International Conference on Machine Learning}, pp.\
  4804--4815. PMLR, 2020.

\bibitem[Jiang et~al.(2019)Jiang, Neyshabur, Mobahi, Krishnan, and
  Bengio]{jiang2019fantastic}
Jiang, Y., Neyshabur, B., Mobahi, H., Krishnan, D., and Bengio, S.
\newblock Fantastic generalization measures and where to find them.
\newblock In \emph{International Conference on Learning Representations}, 2019.

\bibitem[Karakida et~al.(2019)Karakida, Akaho, and
  Amari]{karakida2019normalization}
Karakida, R., Akaho, S., and Amari, S.-i.
\newblock The normalization method for alleviating pathological sharpness in
  wide neural networks.
\newblock In \emph{Advances in Neural Information Processing Systems},
  volume~32. Curran Associates, Inc., 2019.

\bibitem[Kendall(1938)]{kendall}
Kendall, M.~G.
\newblock A new measure of rank correlation.
\newblock \emph{Biometrika}, 30\penalty0 (1/2):\penalty0 81--93, 1938.

\bibitem[Keskar et~al.(2017)Keskar, Nocedal, Tang, Mudigere, and
  Smelyanskiy]{keskar2017large}
Keskar, N.~S., Nocedal, J., Tang, P. T.~P., Mudigere, D., and Smelyanskiy, M.
\newblock On large-batch training for deep learning: Generalization gap and
  sharp minima.
\newblock In \emph{5th International Conference on Learning Representations,
  ICLR 2017}, 2017.

\bibitem[Kingma \& Ba(2015)Kingma and Ba]{kingma2015adam}
Kingma, D.~P. and Ba, J.
\newblock Adam: A method for stochastic optimization.
\newblock In \emph{ICLR (Poster)}, 2015.

\bibitem[Krizhevsky et~al.(2009)Krizhevsky, Nair, and
  Hinton]{krizhevsky2009cifar}
Krizhevsky, A., Nair, V., and Hinton, G.
\newblock {CIFAR}-10 and {CIFAR}-100 datasets.
\newblock 2009.
\newblock URL \url{https://www.cs.toronto.edu/~kriz/cifar.html}.

\bibitem[Laurent \& Massart(2000)Laurent and Massart]{laurent}
Laurent, B. and Massart, P.
\newblock Adaptive estimation of a quadratic functional by model selection.
\newblock \emph{Annals of Statistics}, pp.\  1302--1338, 2000.

\bibitem[Li et~al.(2018)Li, Xu, Taylor, Studer, and
  Goldstein]{li2018visualizing}
Li, H., Xu, Z., Taylor, G., Studer, C., and Goldstein, T.
\newblock Visualizing the loss landscape of neural nets.
\newblock In \emph{NIPS'18: Proceedings of the 32nd International Conference on
  Neural Information Processing Systems}, pp.\  6391--6401. Curran Associates
  Inc., 2018.

\bibitem[Liang et~al.(2019)Liang, Poggio, Rakhlin, and Stokes]{liang2019fisher}
Liang, T., Poggio, T., Rakhlin, A., and Stokes, J.
\newblock Fisher-rao metric, geometry, and complexity of neural networks.
\newblock In \emph{The 22nd International Conference on Artificial Intelligence
  and Statistics}, pp.\  888--896. PMLR, 2019.

\bibitem[Loshchilov \& Hutter(2016)Loshchilov and Hutter]{loshchilov2016sgdr}
Loshchilov, I. and Hutter, F.
\newblock Sgdr: Stochastic gradient descent with warm restarts.
\newblock \emph{arXiv preprint arXiv:1608.03983}, 2016.

\bibitem[McAllester(1999)]{mcallester1999pac}
McAllester, D.~A.
\newblock {PAC-B}ayesian model averaging.
\newblock In \emph{Proceedings of the twelfth annual conference on
  Computational learning theory}, pp.\  164--170, 1999.

\bibitem[Mobahi(2016)]{mobahi2016training}
Mobahi, H.
\newblock Training recurrent neural networks by diffusion.
\newblock \emph{arXiv preprint arXiv:1601.04114}, 2016.

\bibitem[M\"{u}ller et~al.(2019)M\"{u}ller, Kornblith, and
  Hinton]{NEURIPS2019_f1748d6b}
M\"{u}ller, R., Kornblith, S., and Hinton, G.~E.
\newblock When does label smoothing help?
\newblock In Wallach, H., Larochelle, H., Beygelzimer, A., d\textquotesingle
  Alch\'{e}-Buc, F., Fox, E., and Garnett, R. (eds.), \emph{Advances in Neural
  Information Processing Systems}, volume~32. Curran Associates, Inc., 2019.

\bibitem[Nesterov(1983)]{nesterov1983method}
Nesterov, Y.~E.
\newblock A method for solving the convex programming problem with convergence
  rate ${O}(1/k^2)$.
\newblock In \emph{Dokl. akad. nauk Sssr}, volume 269, pp.\  543--547, 1983.

\bibitem[Neyshabur et~al.(2017)Neyshabur, Bhojanapalli, Mcallester, and
  Srebro]{neyshabur2017exploring}
Neyshabur, B., Bhojanapalli, S., Mcallester, D., and Srebro, N.
\newblock Exploring generalization in deep learning.
\newblock In Guyon, I., Luxburg, U.~V., Bengio, S., Wallach, H., Fergus, R.,
  Vishwanathan, S., and Garnett, R. (eds.), \emph{Advances in Neural
  Information Processing Systems}, volume~30. Curran Associates, Inc., 2017.

\bibitem[Rooyen et~al.(2015)Rooyen, Menon, and Williamson]{van2015learning}
Rooyen, B.~v., Menon, A.~K., and Williamson, R.~C.
\newblock Learning with symmetric label noise: the importance of being
  unhinged.
\newblock In \emph{Proceedings of the 28th International Conference on Neural
  Information Processing Systems-Volume 1}, pp.\  10--18, 2015.

\bibitem[Simonyan \& Zisserman(2015)Simonyan and Zisserman]{simonyan2014very}
Simonyan, K. and Zisserman, A.
\newblock Very deep convolutional networks for large-scale image recognition.
\newblock In \emph{3rd International Conference on Learning Representations,
  {ICLR} 2015, San Diego, CA, USA, May 7-9, 2015, Conference Track
  Proceedings}, 2015.

\bibitem[Sun et~al.(2020)Sun, Zhang, Ren, Luo, and
  Li]{DBLP:journals/corr/abs-2006-05620}
Sun, X., Zhang, Z., Ren, X., Luo, R., and Li, L.
\newblock Exploring the vulnerability of deep neural networks: A study of
  parameter corruption.
\newblock \emph{CoRR}, abs/2006.05620, 2020.

\bibitem[Tsuzuku et~al.(2020)Tsuzuku, Sato, and
  Sugiyama]{tsuzuku2020normalized}
Tsuzuku, Y., Sato, I., and Sugiyama, M.
\newblock Normalized flat minima: Exploring scale invariant definition of flat
  minima for neural networks using {PAC-B}ayesian analysis.
\newblock In \emph{International Conference on Machine Learning}, pp.\
  9636--9647. PMLR, 2020.

\bibitem[Uhlmann(2018)]{uhlmann2018generalized}
Uhlmann, J.
\newblock A generalized matrix inverse that is consistent with respect to
  diagonal transformations.
\newblock \emph{SIAM Journal on Matrix Analysis and Applications}, 39\penalty0
  (2):\penalty0 781--800, 2018.

\bibitem[Vaswani et~al.(2017)Vaswani, Shazeer, Parmar, Uszkoreit, Jones, Gomez,
  Kaiser, and Polosukhin]{vaswani2017attention}
Vaswani, A., Shazeer, N., Parmar, N., Uszkoreit, J., Jones, L., Gomez, A.~N.,
  Kaiser, L., and Polosukhin, I.
\newblock Attention is all you need.
\newblock In \emph{NIPS}, 2017.

\bibitem[Xie et~al.(2017)Xie, Girshick, Doll{\'a}r, Tu, and
  He]{xie2017aggregated}
Xie, S., Girshick, R., Doll{\'a}r, P., Tu, Z., and He, K.
\newblock Aggregated residual transformations for deep neural networks.
\newblock In \emph{Proceedings of the IEEE conference on computer vision and
  pattern recognition}, pp.\  1492--1500, 2017.

\bibitem[Yamada et~al.(2019)Yamada, Iwamura, Akiba, and
  Kise]{yamada2019shakedrop}
Yamada, Y., Iwamura, M., Akiba, T., and Kise, K.
\newblock Shakedrop regularization for deep residual learning.
\newblock \emph{IEEE Access}, 7:\penalty0 186126--186136, 2019.

\bibitem[Yi et~al.(2019)Yi, Meng, Chen, Ma, and Liu]{yi2019positively}
Yi, M., Meng, Q., Chen, W., Ma, Z.-m., and Liu, T.-Y.
\newblock Positively scale-invariant flatness of relu neural networks.
\newblock \emph{arXiv preprint arXiv:1903.02237}, 2019.

\bibitem[Yue et~al.(2020)Yue, Nouiehed, and Kontar]{yue2020salr}
Yue, X., Nouiehed, M., and Kontar, R.~A.
\newblock Salr: Sharpness-aware learning rates for improved generalization.
\newblock \emph{arXiv preprint arXiv:2011.05348}, 2020.

\bibitem[Yun et~al.(2019)Yun, Han, Oh, Chun, Choe, and Yoo]{yun2019cutmix}
Yun, S., Han, D., Oh, S.~J., Chun, S., Choe, J., and Yoo, Y.
\newblock Cut{M}ix: Regularization strategy to train strong classifiers with
  localizable features.
\newblock In \emph{Proceedings of the IEEE/CVF International Conference on
  Computer Vision}, pp.\  6023--6032, 2019.

\bibitem[Zagoruyko \& Komodakis(2016)Zagoruyko and
  Komodakis]{zagoruyko2016wide}
Zagoruyko, S. and Komodakis, N.
\newblock Wide residual networks.
\newblock \emph{CoRR}, abs/1605.07146, 2016.

\end{thebibliography}
\bibliographystyle{icml2021}

\appendix
\onecolumn
\section{Proofs}

\subsection{Proof of Theorem \ref{thm2}} \label{app1}

We first introduce the following concentration inequality.

\begin{lem}\label{concent}
Let $\{\epsilon_i, i = 1,\ldots,k\}$ be independent normal variables with mean $0$ and variance $\sigma_i^2$.
Then,
\[P\left(\sum_{i=1}^k \epsilon_i^2 \geq k\sigma_{\max}^2 \left(1+\sqrt{\frac{\log n}{k}}\right)^2 \right) \leq \frac{1}{\sqrt{n}}\]
where $\sigma_{\max} = \max\{\sigma_i\}$.
\end{lem}
\begin{proof}
From Lemma 1 in \citet{laurent}, for any $x > 0$,
\[P\left(\sum_{i=1}^k \epsilon_i^2 \geq \sum_{i=1}^k \sigma_i^2 + 2\sqrt{\sum_{i=1}^k \sigma_i^4 x} + 2\sigma_{\max}^2 x\right) \leq \exp(-x).\]
Since
\begin{align*}
\sum_{i=1}^k \sigma_i^2 + 2\sqrt{\sum_{i=1}^k \sigma_i^4 x} + 2\sigma_{\max}^2 x &\leq \sigma_{\max}^2(k+2\sqrt{kx}+2x)\\
&\leq \sigma_{\max}^2(\sqrt{k}+\sqrt{2x})^2,
\end{align*}
plugging $x=\frac{1}{2}\log n$ proves the lemma.

\end{proof}

\begin{thm}
Let $T^{-1}_\mf{w}$ be a normalization operator of $\mf{w}$ on $\mathbb{R}^k$.
If $L_D(\mf w) \leq E_{\epsilon_i \sim \mathcal{N}(0, \sigma^2)} [L_D(\mf w+\bm\epsilon)]$ for some $\sigma > 0$,
then with probability $1-\delta$, 
\ificml
    \[L_D(\mf{w}) \leq \max_{\Vert T^{-1}_\mf{w} \epsilon\Vert_2 \leq \rho} L_S(\mf w+\bm\epsilon)  + \sqrt{\frac{1}{n-1}\left(k\log\left(1 + \frac{\Vert \mf w\Vert_2^2}{\eta^2\rho^2} \left(1+\sqrt{\frac{\log n}{k}}\right)^2 \right) + 4\log\frac{n}{\delta} + O(1)\right)}\]
\else
    \begin{align}
        \nonumber L_D(\mf{w}) \leq & \max_{\Vert T^{-1}_\mf{w} \epsilon\Vert_2 \leq \rho} L_S(\mf w+\bm\epsilon) \\
        \nonumber & + \sqrt{\frac{1}{n-1}\left(k\log\left(1 + \frac{\Vert \mf w\Vert_2^2}{\eta^2\rho^2} \left(1+\sqrt{\frac{\log n}{k}}\right)^2 \right) + 4\log\frac{n}{\delta} + O(1)\right)}
    \end{align}
\fi
where $n = \vert S\vert$ and $\rho = \sqrt{k}\sigma (1+\sqrt{\log n / k}) /\eta$.
\end{thm}
\begin{proof}
The idea of the proof is given in \citet{foret2021sharpnessaware}. From the assumption, adding Gaussian perturbation on the weight space does not improve the test error.
Moreover, from Theorem 3.2 in \citet{chatterji2019intriguing}, the following generalization bound holds under the perturbation:
\ificml
    \[ E_{\epsilon_i \sim \mathcal{N}(0, \sigma^2)} [L_D(\mf w+\bm\epsilon)]
    \leq E_{\epsilon_i \sim \mathcal{N}(0, \sigma^2)} [L_S(\mf w+\bm\epsilon)]
    + \sqrt{\frac{1}{n-1}\left(\frac{1}{4}k\log\left(1 + \frac{\Vert\mf{w}\Vert_2^2}{k\sigma^2}\right) + \log\frac{n}{\delta} + C(n, \sigma, k)\right)}.\]
\else
    \begin{align}
        \nonumber E_{\epsilon_i \sim \mathcal{N}(0, \sigma^2)} [L_D(\mf w+\bm\epsilon)] &
        \leq E_{\epsilon_i \sim \mathcal{N}(0, \sigma^2)} [L_S(\mf w+\bm\epsilon)] \\
        \nonumber &+ \sqrt{\frac{1}{n-1}\left(\frac{1}{4}k\log\left(1 + \frac{\Vert\mf{w}\Vert_2^2}{k\sigma^2}\right) + \log\frac{n}{\delta} + C(n, \sigma, k)\right)}.
    \end{align}
\fi
Therefore, the left hand side of the statement can be bounded as
\ificml
    \begin{align*}
        L_D(\mf w) &\leq E_{\epsilon_i \sim \mathcal{N}(0, \sigma^2)} [L_S(\mf w+\bm\epsilon)] + \sqrt{\frac{1}{n-1}\left(\frac{1}{4}k\log\left(1 + \frac{\Vert\mf{w}\Vert_2^2}{k\sigma^2}\right) + \log\frac{n}{\delta} + C\right)}\\
        &\leq \left(1-\frac{1}{\sqrt{n}}\right) \max_{\Vert T^{-1}_\mf{w} \epsilon\Vert_2 \leq \rho} L_S(\mf w+\bm\epsilon)
        + \frac{1}{\sqrt{n}} + \sqrt{\frac{1}{n-1}\left(\frac{1}{4}k\log\left(1 + \frac{\Vert\mf{w}\Vert_2^2}{k\sigma^2}\right) + \log\frac{n}{\delta} + C\right)}\\
        &\leq \max_{\Vert T^{-1}_\mf{w} \epsilon\Vert_2 \leq \rho} L_S(\mf w+\bm\epsilon)
        + \sqrt{\frac{1}{n-1}\left(k\log\left(1 + \frac{\Vert\mf{w}\Vert_2^2}{k\sigma^2}\right) + 4\log\frac{n}{\delta} + 4C\right)}
    \end{align*}
\else
    \begin{align*}
        L_D(\mf w) &\leq E_{\epsilon_i \sim \mathcal{N}(0, \sigma^2)} [L_S(\mf w+\bm\epsilon)] \\
        &+ \sqrt{\frac{1}{n-1}\left(\frac{1}{4}k\log\left(1 + \frac{\Vert\mf{w}\Vert_2^2}{k\sigma^2}\right) + \log\frac{n}{\delta} + C\right)}\\
        &\leq \left(1-\frac{1}{\sqrt{n}}\right) \max_{\Vert T^{-1}_\mf{w} \epsilon\Vert_2 \leq \rho} L_S(\mf w+\bm\epsilon) 
        + \frac{1}{\sqrt{n}} \\
        &+ \sqrt{\frac{1}{n-1}\left(\frac{1}{4}k\log\left(1 + \frac{\Vert\mf{w}\Vert_2^2}{k\sigma^2}\right) + \log\frac{n}{\delta} + C\right)}\\
        &\leq \max_{\Vert T^{-1}_\mf{w} \epsilon\Vert_2 \leq \rho} L_S(\mf w+\bm\epsilon)
        + \sqrt{\frac{1}{n-1}\left(k\log\left(1 + \frac{\Vert\mf{w}\Vert_2^2}{k\sigma^2}\right) + 4\log\frac{n}{\delta} + 4C\right)}
    \end{align*}
\fi
where the second inequality follows from Lemma \ref{concent} and  $\Vert T^{-1}_\mf{w}\Vert_2 \leq \frac{1}{\eta}$.

\end{proof}

\section{Correlation Analysis} \label{corrdetail}

To capture the correlation between generalization measures, i.e., sharpness and adaptive sharpness, and actual generalization gap, we utilize Kendall rank correlation coefficient \cite{kendall}. Formally, given the set of pairs of a measure and generalization gap observed $S=\{(m_1, g_1), \ldots, (m_n, g_n)\}$, Kendall rank correlation coefficient $\tau$ is given by
\[\tau(S)=\frac{2}{n(n-1)}\sum_{i<j} \mathrm{sign}(m_i-m_j)\mathrm{sign}(g_i-g_j).\]
Since $\tau$ represents the difference between the proportion of concordant pairs, i.e., either both $m_i < m_j$ and $g_i < g_j$ or both $m_i > m_j$ and $g_i > g_j$ among the whole $\binom{n}{2}$ point pairs, and the proportion of discordant pairs, i.e., not concordant, the value of $\tau$ is in the range of $[-1, 1]$.

While the rank correlation coefficient aggregates the effects of all the hyper-parameters, granulated coefficient \cite{jiang2019fantastic} can consider the correlation with respect to the each hyper-parameter separately. If $\Theta=\prod_{i=1}^N \Theta_i$ is the Cartesian product of each hyper-parameter space $\Theta_i$, granulated coefficient with respect to $\Theta_i$ is given by
\[\psi_i=\frac{1}{\vert\Theta_{-i}\vert} \sum_{\theta_1 \in \Theta_1} \cdots \sum_{\theta_{i-1} \in \Theta_{i-1}} \sum_{\theta_{i+1} \in \Theta_{i+1}} \cdots \sum_{\theta_N \in \Theta_N} \tau\left(\bigcup_{\theta_i \in \Theta_i}\{(m(\bm\theta), g(\bm\theta)\}\right)\]
where $\Theta_{-i}=\Theta_1 \times \cdots \Theta_{i-1} \times \Theta_{i+1} \times \Theta_N$. Then the average $\Psi = \sum_{i=1}^N \psi_i / N$ of $\psi_i$ indicates whether the correlation exists across all hyper-parameters.

We vary $4$ hyper-parameters, mini-batch size, initial learning rate, weight decay coefficient and dropout rate, to produce different models. It is worth mentioning that changing one or two hyper-parameters for correlation analysis may cause spurious correlation \cite{jiang2019fantastic}. For each hyper-parameter, we use $5$ different values in Table~\ref{hyper} which implies that $5^4=625$ configurations in total.

\begin{table}[h]
\centering
\captionsetup{justification=centering}
\caption{Hyper-parameter configurations. \label{hyper}}
\begin{center}
\begin{small}
\begin{tabular}{lc}
\toprule
mini-batch size & $32,~64,~128,~256,~512$\\
learning rate & $0.0033,~0.01,~0.033,~0.1,~0.33$\\
weight decay & $5\mathrm{e}{-7},~5\mathrm{e}{-6},~5\mathrm{e}{-5},~5\mathrm{e}{-4},~5\mathrm{e}{-3}$\\
dropout rate & $0,~0.125,~0.25,~0.375,~0.5$\\
\bottomrule
\end{tabular}
\end{small}
\end{center}
\end{table}

By using the above hyper-parameter configurations, we train WideResNet-28-2 model on CIFAR-10 dataset. We use SGD as an optimizer and set momentum to $0.9$. We set the number of epochs to $200$ and cosine learning rate decay \citep{loshchilov2016sgdr} is adopted. Also, random resize, padding by four pixels, normalization and random horizontal flip are applied for data augmentation and label smoothing \citep{NEURIPS2019_f1748d6b} is adopted with its factor of $0.1$. 
Using model parameters with training accuracy higher than $99.0\%$ among the generated models, we calculate sharpness and adaptive sharpness with respect to generalization gap. 

To calculate adaptive sharpness, we fix normalization scheme to element-wise normalization. We calculate adaptive sharpness and sharpness with both $p=2$ and $p=\infty$. We conduct a grid search over \{$5\mathrm{e}{-6}$, $1\mathrm{e}{-5}$, $5\mathrm{e}{-5}$, \ldots, $5\mathrm{e}{-1}$, $1.0$\} to obtain each $\rho$ for sharpness and adaptive sharpness which maximizes correlation with generalization gap. As results of the grid search, we select $1\mathrm{e}{-5}$ and $5\mathrm{e}{-4}$ as $\rho$s for sharpness of $p=2$ and $p=\infty$, respectively, and select $5\mathrm{e}{-1}$ and $5\mathrm{e}{-3}$ as $\rho$s for adaptive sharpness of $p=2$ and $p=\infty$, respectively. To calculate maximizers of each loss function for calculation of sharpness and adaptive sharpness, we follow $m$-sharpness strategy suggested by \citet{foret2021sharpnessaware} and $m$ is set to $8$.

\end{document}